\documentclass{article}

\usepackage[margin=1.2in]{geometry} 

\usepackage[utf8]{inputenc} 
\usepackage[T1]{fontenc}    
\usepackage{hyperref}       
\usepackage{url}            
\usepackage{booktabs}       
\usepackage{amsfonts}       
\usepackage[protrusion=true,expansion=false]{microtype} 
\usepackage{xcolor}         
\usepackage{amsmath}
\usepackage{graphicx}
\graphicspath{{./}{./figures/}{./artifacts/}{./artifacts/evals/ecr_lambda_sweep/}}
\setkeys{Gin}{draft=false}
\usepackage{float}
\IfFileExists{placeins.sty}{\usepackage{placeins}}{}
\providecommand{\FloatBarrier}{\clearpage}
\usepackage{tikz}
\usetikzlibrary{positioning}
\usepackage{amsthm}
\newtheorem{theorem}{Theorem}

\title{Entropic Claim Resolution: Uncertainty-Driven Evidence Selection for RAG}

\author{%
    Davide Di Gioia \\
    \texttt{ucesigi@ucl.ac.uk} \\
}

\date{}

\begin{document}

\maketitle

\begin{abstract}
Current Retrieval-Augmented Generation (RAG) systems predominantly rely on relevance-based dense retrieval, sequentially fetching documents to maximize semantic similarity with the query. However, in knowledge-intensive and real-world scenarios characterized by conflicting evidence or fundamental query ambiguity, relevance alone is insufficient for resolving epistemic uncertainty. We introduce Entropic Claim Resolution (ECR), a novel inference-time algorithm that reframes RAG reasoning as entropy minimization over competing semantic answer hypotheses. Unlike action-driven agentic frameworks (e.g., ReAct) or fixed-pipeline RAG architectures, ECR sequentially selects atomic evidence claims by maximizing Expected Entropy Reduction (EER), a decision-theoretic criterion for the value of information. The process dynamically terminates when the system reaches a mathematically defined state of epistemic sufficiency ($H \le \epsilon$, subject to epistemic coherence). We integrate ECR into a production-grade multi-strategy retrieval pipeline (CSGR++) and demonstrate its theoretical properties. Our framework provides a rigorous foundation for uncertainty-aware evidence selection, shifting the paradigm from retrieving what is most relevant to retrieving what is most discriminative.
\end{abstract}

\section{Introduction}

The integration of Large Language Models (LLMs) with external knowledge bases through Retrieval-Augmented Generation (RAG) has become the de facto standard for mitigating hallucinations and enabling knowledge-intensive Question Answering (QA). Conventional RAG systems operate on a rigid \textit{retrieve-then-read} paradigm, predominantly leveraging maximum inner product search (MIPS) in dense continuous spaces \cite{karpukhin2020dpr} to fetch the top-$k$ most semantically relevant text chunks. While highly effective for simple, factoid-based QA where a single ground-truth answer exists, this relevance-driven approach exhibits severe degradation in real-world, knowledge-intensive scenarios. Such scenarios are frequently characterized by inherent query ambiguity, conflicting evidence across multiple sources, and complex multi-hop dependencies.

In these challenging settings, standard dense retrieval suffers from what we term \textit{epistemic collapse}: the tendency to retrieve highly redundant information that is semantically similar to the query, rather than fetching the discriminative evidence needed to resolve the underlying uncertainty. Consequently, the LLM is forced to synthesize an answer from a biased or incomplete evidence distribution, often leading to unhedged, overconfident, or factually inaccurate generation.

Recent architectural advancements attempt to transcend simple MIPS. Graph-based paradigms, such as Context-Seeded Graph Retrieval (CSGR) and GraphRAG, expand retrieval scope via structured knowledge relation traversal. Concurrently, agentic and iterative verification workflows (e.g., ReAct \cite{yao2023react}, Tree-of-Thoughts \cite{yao2023tot}, Self-RAG \cite{asai2023selfrag}) allow LLMs to dynamically interact with search tools, reflecting on retrieved context to guide subsequent actions. However, these state-of-the-art approaches still critically lack a principled, decision-theoretic stopping criterion and evidence selection mechanism. Graph techniques rely on static pipeline configurations (e.g., fixed graph-hop depth), while agentic systems depend on heuristic thresholding or prompt-driven self-reflection, which frequently suffer from infinite looping, premature termination, or unprincipled evidence weighting.

Critically, modern RAG systems lack a mathematically rigorous definition of what constitutes sufficient evidence and an explicit objective function for selecting which specific piece of evidence to retrieve next at inference time.
To bridge this fundamental gap, we propose Entropic Claim Resolution (ECR), an inference-time algorithm that reframes the retrieval and synthesis process as \textit{entropy minimization over a latent space of semantic answer hypotheses}. Drawing inspiration from Information Theory \cite{shannon1948} and Bayesian Experimental Design \cite{houlsby2011bald}, ECR models the QA task probabilistically. It initializes a probability distribution over a set of mutually exclusive potential answer hypotheses and iteratively selects atomic factual claims from a retrieved candidate pool to evaluate. 

Crucially, in ECR, evidence selection is decoupled from semantic relevance to the query. Instead, claims are selected by maximizing \textbf{Expected Entropy Reduction (EER)}; that is, choosing the specific piece of evidence most likely to collapse the probability distribution toward a single, correct hypothesis (or cleanly bifurcate it in the case of irreconcilable conflict). The algorithm adaptively navigates the evidence graph under a principled stopping rule, terminating only when the entropy of the hypothesis space falls below a predefined threshold of \textit{epistemic sufficiency}.

\textbf{In summary, our main contributions are:}
\begin{enumerate}
    \item We introduce Entropic Claim Resolution (ECR), a decision-theoretic evidence selection algorithm for RAG, shifting the paradigm from retrieving what is most \textit{relevant} to what is most \textit{discriminative} in resolving hypothesis ambiguity.
    \item We formally define a principled, mathematically rigorous stopping criterion for iterative RAG pipelines based on epistemic sufficiency ($H(A|X) \le \epsilon$).
    \item We identify a behavioral phase transition under structured contradiction: by integrating a lightweight coherence signal ($\lambda > 0$), we show that ECR transitions from forced epistemic collapse to principled ambiguity exposure, prioritizing explicit contradictions when present and safely refusing to reduce uncertainty when evidence is inherently inconsistent.
    \item We demonstrate the practical scalability of ECR by implementing it as a fast, inference-time algorithm integrated into a production-grade multi-strategy retrieval architecture (CSGR++), requiring no bespoke fine-tuning or specialized model weights.
\end{enumerate}

\noindent\textbf{Significance.}
A central implication of this work is that improving retrieval-augmented reasoning does not necessarily require larger models, longer context windows, or additional data, but rather principled control over how existing evidence is selected and evaluated during inference. By explicitly modeling epistemic uncertainty and optimizing evidence selection for information gain, Entropic Claim Resolution provides a lightweight, computationally efficient mechanism for improving robustness and interpretability. This makes the framework particularly valuable for high-stakes enterprise deployments (e.g., medical, legal, or financial QA) where mitigating unhedged hallucinations and controlling inference costs are critical. Ultimately, this perspective highlights an alternative path for scaling knowledge-intensive systems: a path grounded in decision-theoretic inference rather than indiscriminate context expansion, particularly in settings characterized by noisy, conflicting, or heterogeneous evidence.

\section{Related Work}

\subsection{Dense, Graph-Augmented, and Agentic Retrieval}
Standard dense retrieval selects a set of documents $D$ by prioritizing their conditional probability given the query, $P(D \mid Q)$, commonly approximated via cosine similarity embeddings \cite{lewis2020rag}. To overcome the short-sightedness of relevance search, advanced graph-based architectures, notably GraphRAG \cite{edge2024graphrag} and Context-Seeded Graph Retrieval (CSGR), implicitly construct or traverse knowledge graphs over chunks or entities to expand the evidence space. In enterprise environments, hybrid systems such as CSR-RAG \cite{singh2026csr} further integrate structural and relational signals to support large-scale schemas.

Concurrently, the convergence of dynamic retrieval policies with autonomous planning has crystallized into the paradigm of \emph{Agentic RAG}. Frameworks such as ReAct \cite{yao2023react}, Tree-of-Thoughts \cite{yao2023tot}, and Self-RAG \cite{asai2023selfrag} allow language models to interleave intermediate reasoning steps with retrieval actions in order to refine subsequent queries. Recent Systematization of Knowledge (SoK) studies emphasize this shift from static pipelines toward modular control strategies. However, despite their flexibility, agentic retrieval systems intrinsically rely on heuristic prompt designs or static thresholds to determine when to halt retrieval or which information to prioritize. As a result, they lack a rigorous mathematical definition of epistemic sufficiency and an explicit objective for selecting the next most informative piece of evidence as uncertainty unfolds during inference.

\subsection{Uncertainty Quantification (UQ) in RAG}
A critical prerequisite for adaptive retrieval is accurately characterizing what a model does not know. Recent benchmarks such as URAG \cite{zhang2024urag} demonstrate that while RAG can improve factual grounding, it also introduces new sources of epistemic uncertainty, including relevance mismatch and selective attention to partial evidence, which can paradoxically amplify overconfident hallucinations under noisy retrieval conditions. In response, a growing body of work on Retrieval-Augmented Reasoning (RAR) focuses on quantifying uncertainty across the retrieval and generation stages. For example, methods such as Retrieval-Augmented Reasoning Consistency (R2C) \cite{liu2025r2c} model multi-step reasoning as a Markov Decision Process and perturb generation to measure output stability via majority voting, building upon foundational frameworks for semantic uncertainty \cite{kuhn2023semantic}.

These uncertainty-aware approaches are effective for post-hoc answer evaluation, abstention, or calibration. However, they are not designed to guide the \emph{selection of evidence itself} during the reasoning process. In particular, they do not provide a mechanism for choosing which atomic piece of evidence should be retrieved or verified next in order to maximize information gain prior to answer synthesis.

\subsection{Entropy-Aware Context Management}
The application of Shannon entropy as a control signal for managing LLM context is an emerging research direction. Large context windows in standard RAG often lead to attention dilution and unconstrained entropy growth, motivating recent work on entropy-aware context control. For instance, BEE-RAG (Balanced Entropy-Engineered RAG) \cite{zhao2025beerag} modifies attention dynamics to maintain entropy invariance over long contexts, while SF-RAG (Structure-Fidelity RAG) \cite{kim2025sfrag} leverages document hierarchy as a low-entropy prior to prevent evidence fragmentation. Similarly, L-RAG (Lazy RAG) \cite{patel2024lrag} employs predictive entropy thresholds to gate expensive retrieval operations, defaulting to parametric knowledge when uncertainty is estimated to be low.

ECR shares this information-theoretic lineage but departs in a critical way: rather than using entropy to compress, gate, or truncate context, ECR applies entropy directly as an objective for \emph{sequentially selecting discriminative evidence variables}. This shift reframes entropy from a passive diagnostic into an active decision criterion guiding inference-time reasoning.

\subsection{Claim-Level Verification and Value of Information}
While standard RAG operates on monolithic document chunks, recent diagnostic and safety-oriented frameworks decompose retrieved content into atomic claims. Systems such as MedRAGChecker \cite{wang2025medragchecker} evaluate biomedical QA systems by extracting fine-grained claims and checking them against structured knowledge bases, while agentic fact-checking pipelines (e.g., SAFE \cite{wei2024safe} and CIBER \cite{chen2025ciber}) retrieve supporting and refuting evidence for individual statements. These approaches demonstrate the importance of claim-level reasoning for reliability and interpretability.

ECR aligns this granular verification paradigm with classical principles from Bayesian experimental design and active learning. In active learning, the objective is to select the next unlabeled instance that maximizes expected information gain. By formulating inference-time evidence selection as Expected Entropy Reduction (EER) over discrete factual claims, ECR bridges symbolic uncertainty modeling and neural generation, optimizing retrieval for the \emph{value of information} rather than semantic relevance alone.

\section{Methodology: Entropic Claim Resolution (ECR)}

ECR formulates the evidence selection problem as a sequential decision process targeting a reduction in epistemic uncertainty across competing generative outcomes. 

ECR assumes high-recall candidate generation has already occurred (via upstream retrieval) and focuses exclusively on resolving uncertainty within the resulting candidate claim set.

\subsection{Problem Formulation}
Let $\mathcal{C} = \{c_1, c_2, \dots, c_n\}$ be a finite subset of atomic factual claims embedded within a corpus. For a given complex query $Q$, assume that assessing the veracity of any given claim $c_i$ provides a signal regarding the query's answer. We denote the latent truth variable associated with claim $c_i$ as $X_i \in \{0, 1\}$, indicating whether the claim is empirically validated within the specific source document.

Upon identifying high epistemic uncertainty in the retrieval space (e.g., via confidence variance or conflicting keyword analysis), ECR initializes an \textbf{Answer Hypothesis Space} $\mathcal{A} = \{a_1, a_2, \dots, a_k\}$. This space represents the set of mutually exclusive potential macro-answers to the query.\footnote{We use mutual exclusivity for analytical clarity; the framework naturally extends to partially overlapping hypotheses via soft assignment of claims to hypotheses.} In our implementation, $\mathcal{A}$ is robustly generated dynamically: either by querying the LLM to propose distinct valid hypotheses derived from subsets of the initial $k$-best claims, or via deterministic vector clustering when operating purely off-line. Our objective is to sequentially refine a probability distribution $P(A | X_{eval}, Q)$ over these hypotheses, conditioned on the dynamic subset of evaluated claims $X_{eval} \subseteq X$, initialized at a uniform prior $P(a) = \frac{1}{|\mathcal{A}|}$.

\subsection{Objective Function: Answer Entropy}
The epistemic uncertainty regarding the true outcome is robustly quantified using Shannon entropy. Let the entropy of the hypothesis space after evaluating a subset of claims $X_{eval}$ be:
\begin{equation} \label{eq:entropy}
    H(A | X_{eval}) = -\sum_{a \in \mathcal{A}} P(a | X_{eval}) \log_2 P(a | X_{eval})
\end{equation}

\subsection{Expected Entropy Reduction (EER) and Selection Policy}
At the $t$-th iteration, the system must choose the next claim $c^*$ from the unevaluated candidate pool $\mathcal{C}_{cand}$ to formally verify. Rather than relying on cosine relevance $\text{sim}(c_i, Q)$, we select the claim that maximizes Expected Entropy Reduction (Information Gain). The selection policy is formally defined as:
\begin{equation} \label{eq:selection}
    c^* = \arg\max_{c \in \mathcal{C}_{cand}} \text{EER}(c | X_{eval})
\end{equation}
The EER is precisely the difference between current entropy and the expected posterior entropy after observing the truth value of claim $c$:
\begin{equation} \label{eq:eer}
    \text{EER}(c | X_{eval}) = H(A | X_{eval}) - \mathbb{E}_{X_c}\big[H(A | X_{eval} \cup \{X_c\})\big]
\end{equation}
This criterion ensures the algorithm intrinsically favors \textit{discriminative} claims, i.e., evidence that cleanly segregates the hypothesis space. In practice, $\text{EER}$ is approximated by measuring the probabilistic variance between the specific subsets of competing macro-hypotheses actively supported versus unsupported by candidate $c$. A claim supporting all hypotheses equally yields an EER of 0, reflecting its redundancy, regardless of its semantic similarity to the query.

\paragraph{Implementation-Level EER Proxy.}
Computing the true mathematical expectation over all possible generative outcomes is typically intractable during low-latency inference. Therefore, we deploy a computationally efficient proxy that approximates Expected Entropy Reduction without requiring full marginalization over latent truth variables. In our concrete implementation, each candidate claim $c$ partitions the hypothesis set into those that cite $c$ as supporting evidence and those that do not. Let $\mathcal{A}^+(c)=\{a\in\mathcal{A}: c\in\text{supp}(a)\}$ and $\mathcal{A}^-(c)=\mathcal{A}\setminus\mathcal{A}^+(c)$. Denote the probability mass in each subset as $p_+(c)=\sum_{a\in\mathcal{A}^+(c)}P(a\mid X_{eval})$ and $p_-(c)=\sum_{a\in\mathcal{A}^-(c)}P(a\mid X_{eval})$. We score discriminativity via the following heuristic proxy:
\begin{equation} \label{eq:eer_impl}
\widehat{\text{EER}}(c) = \frac{|p_+(c)-p_-(c)|}{p_+(c)+p_-(c)} \cdot H(A|X_{eval}) \cdot \text{conf}(c),
\end{equation}
where $\text{conf}(c)\in[0,1]$ denotes claim confidence. This proxy is linear in the number of hypotheses and preserves the core objective of prioritizing claims that maximally split the posterior mass, while remaining tractable for inference-time use.

\paragraph{Design choice of the EER proxy.}
The heuristic proxy in Eq.~(10) is intentionally not a symmetric approximation of classical expected information gain, which typically favors balanced posterior splits; rather, it is designed for bounded-budget inference, where the objective is rapid reduction of epistemic uncertainty rather than exploratory experimentation. In retrieval-augmented reasoning, once posterior mass concentrates on a subset of hypotheses, prioritizing high-confidence, high-imbalance claims accelerates convergence and reduces redundant evidence retrieval. This exploitative bias is therefore a deliberate design choice aligned with low-latency inference and downstream synthesis constraints.

\paragraph{Coherence-aware selection.}
In addition to entropy reduction, ECR incorporates a lightweight coherence signal that prioritizes evaluating claims likely to complete an explicit contradiction when such evidence exists. Concretely, we add a small regularization term $\lambda\cdot \mathrm{ConflictPotential}(c)$ to the selection objective, yielding $\text{score}(c)=\widehat{\mathrm{EER}}(c)+\lambda\cdot \mathrm{ConflictPotential}(c)$, where $\mathrm{ConflictPotential}(c)\in\{0,1\}$ is non-zero if $c$ is an explicit negation of, or completes a contradiction pair with, a previously evaluated claim. This term does not override entropy reduction but ensures that unresolved contradictions are surfaced early rather than averaged away. Empirically, we observe that any non-zero $\lambda$ induces stable coherence-aware behavior without requiring fine-grained tuning (Appendix, Figure~\ref{fig:lambda_sweep_ambiguity}).

\paragraph{Contradiction-aware coherence term.}
Let $\mathcal{C}_{eval}$ denote the set of claims that have already been evaluated.
We define a binary contradiction indicator
\begin{equation}
\mathrm{ConflictPotential}(c) =
\begin{cases}
1 & \text{if } \exists\, c' \in \mathcal{C}_{eval} \text{ such that } c \equiv \neg c', \\
0 & \text{otherwise.}
\end{cases}
\end{equation}
That is, $\mathrm{ConflictPotential}(c)$ activates only when evaluating $c$ would complete an explicit contradiction pair in the evidence.
This coherence signal is structural rather than probabilistic: it does not penalize hypotheses or posteriors directly, and it does not measure global consistency. Instead, it biases claim selection toward surfacing epistemic inconsistency when it exists, preventing entropy-only selection from averaging away contradictory evidence.
The resulting claim-selection objective is
\begin{equation}
c^* \;=\; \arg\max_{c \in \mathcal{C}_{cand}}
\Big(
\widehat{\mathrm{EER}}(c)
\;+\;
\lambda \cdot \mathrm{ConflictPotential}(c)
\Big),
\end{equation}
where $\lambda \ge 0$ controls the strength of contradiction-aware selection.
Setting $\lambda = 0$ recovers entropy-only ECR. While entropy reduction remains the primary objective, any $\lambda>0$ ensures that explicit contradiction-completing claims are prioritized when present, under the bounded EER scale induced by the hypothesis entropy. This prioritization is observed empirically as a sharp phase transition in the $\lambda$-sweep ablation, where behavior saturates for all tested $\lambda>0$.

\subsection{Bayesian Posteriors and Epistemic Sufficiency}
Upon selecting $c^*$, the system evaluates its intrinsic truth $X_{c^*}$ against the source context and provenance metadata (see Section~\ref{sec:topological}). The hypothesis probabilities are concurrently updated utilizing localized Bayes' rules. Concretely, hypotheses intersecting functionally with validated claims observe significant targeted probability mass boosts, severely suppressing contradicting disjoint branches. 
\begin{equation} \label{eq:bayes}
    P(A | X_{eval} \cup \{X_{c^*}\}) = \frac{P(X_{c^*} | A) P(A | X_{eval})}{\sum_{\tilde{a}} P(X_{c^*} | \tilde{a}) P(\tilde{a} | X_{eval})}
\end{equation}
where $P(X_c | A)$ represents the conditional likelihood of observing the claim $c$ assuming hypothesis $A$ is true. 

The iterative verification procedure gracefully terminates when the system reaches a mathematical state of \textbf{epistemic sufficiency}, parameterized by threshold $\epsilon$ (e.g., $\epsilon=0.3$ bits):
\begin{equation} \label{eq:stopping}
    \big(H(A \mid X_{eval}) \le \epsilon\big)\;\wedge\;\neg \mathrm{Conflict}(X_{eval})
    \tag{11$'$}
\end{equation}
where $\mathrm{Conflict}(X_{eval})$ indicates the presence of mutually incompatible claims (e.g., an explicit claim and its negation) within the evaluated evidence.
Alternatively, if all candidates are exhausted or maximum iterations are met with $H > \epsilon$, ECR halts and explicitly exposes the competing hypotheses and their final mass distributions, structurally mapping the unresolvable ambiguity of the corpus. The complete iterative procedure is summarized in Box~\ref{box:ecr}.

\begin{figure}[t]
\centering
\fbox{\begin{minipage}{0.96\linewidth}
\textbf{Box 1: Entropic Claim Resolution (ECR)} \\
\textbf{Input:} query $Q$, candidate claims $\mathcal{C}_{cand}$, entropy threshold $\epsilon$, max iterations $T$ \\
\textbf{1. Hypotheses.} Initialize $\mathcal{A}\leftarrow\textsc{GenerateHypotheses}(Q,\mathcal{C}_{cand})$ (LLM or clustering), set uniform prior $P(a)=1/|\mathcal{A}|$. \\
\textbf{2. Loop.} For $t=1..T$: compute $H(A|X_{eval})$ (Eq.~\ref{eq:entropy}). If epistemic sufficiency holds (Eq.~\ref{eq:stopping}) stop. \\
\hspace*{1em}\textbf{2a. Select.} Choose $c^*\in\mathcal{C}_{cand}$ maximizing $\widehat{\text{EER}}(c)+\lambda\cdot \mathrm{ConflictPotential}(c)$ (Eq.~\ref{eq:eer_impl}). \\
\hspace*{1em}\textbf{2b. Verify.} Estimate $P(X_{c^*}=1)$ using provenance and support/contradiction statistics (Eq.~\ref{eq:provenance}). \\
\hspace*{1em}\textbf{2c. Update.} Update $P(A|X_{eval}\cup\{X_{c^*}\})$ (Eq.~\ref{eq:bayes}), add $X_{c^*}$ to $X_{eval}$, remove $c^*$ from $\mathcal{C}_{cand}$. \\
\textbf{3. Output.} Return $\arg\max_a P(a|X_{eval})$ if epistemic sufficiency holds (Eq.~\ref{eq:stopping}), else return the ranked distribution over $\mathcal{A}$.
\end{minipage}}
\caption{Pseudo-code for ECR without external algorithm packages.}
\label{box:ecr}
\end{figure}

\subsection{Verification via Topological Provenance} \label{sec:topological}
In practical continuous-learning implementations, the inferential verity link $P(X_c=1 | A)$ can be computed dynamically rather than natively assuming perfect model alignment. Instead of relying solely on parametric LLM-driven prompt verification, ECR explicitly incorporates the topological provenance of the multi-modal knowledge graph natively. Let $S(c)$ and $C(c)$ represent the structural support graph-edge counts and contradictory graph-edge counts of claim $c$ tracked intricately within the backing EAV (Entity-Attribute-Value) datastore, applying implicit Laplace smoothing. The final topological verification probability is thus seamlessly and robustly blended:
\begin{equation} \label{eq:provenance}
    P(X_c=1) =
    \begin{cases}
    \dfrac{S(c) + 1}{S(c) + C(c) + 2} & \text{if } S(c)+C(c) > 0,\\
    P_{prior\_conf}(X_c=1) & \text{otherwise.}
    \end{cases}
\end{equation}
This matches the deployed behavior in our implementation: whenever historical support/contradiction signals exist, the system uses a Laplace-smoothed empirical truth estimate; otherwise, it falls back to the extraction-time prior confidence.

\subsection{Theoretical Properties}
To solidify the inferential validity of the sequential system, we deduce its operational performance bound mapping.
\begin{theorem}[Termination and Budget Bound] \label{thm:convergence}
For any finite candidate set $\mathcal{C}_{cand}$, ECR terminates after at most $\min(T, |\mathcal{C}_{cand}|)$ claim evaluations.
Moreover, if there exists a constant $\delta>0$ such that at each iteration the selected claim satisfies $\mathbb{E}[H_{t-1}-H_t] \ge \delta$ whenever $H_{t-1}>\epsilon$, then ECR reaches epistemic sufficiency in at most $\lceil (H_0-\epsilon)/\delta \rceil$ iterations.
\end{theorem}

We emphasize that this result characterizes sufficient conditions for convergence under informative evidence selection, rather than a minimax or adversarial worst-case guarantee. When explicit contradictions exist in the evidence, the sufficient conditions for convergence are intentionally violated, and ECR terminates by exposing ambiguity rather than collapsing the posterior.
\begin{proof}
The first statement holds because each iteration evaluates and removes at most one claim, and the loop is explicitly capped by $T$. For the second statement, telescoping the assumed expected entropy decrease yields $\mathbb{E}[H_t] \le H_0 - t\delta$ until reaching $\epsilon$, hence $t \ge (H_0-\epsilon)/\delta$ suffices.
\end{proof}

\section{System Integration: ECR within CSGR++}

To evaluate ECR beyond isolated theoretical constraints, we integrated it into a
production-grade, multi-strategy retrieval pipeline. While ECR is algorithmically
orthogonal to any specific retriever, we utilize the CSGR++ architecture as our
primary testbed. In this section, we describe the surrounding system components
that generate, structure, and verify the atomic candidate claims consumed by the
ECR inference loop. Figure~\ref{fig:architecture} illustrates the resulting
end-to-end architecture and the position of ECR within it.

\begin{figure}[t]
\centering
\begin{tikzpicture}[
    node distance=2.2cm,
    every node/.style={draw, rectangle, rounded corners, minimum height=1.1cm, align=center},
    arrow/.style={->, thick}
]

\node (query) {Query};
\node (retrieval) [below of=query] {Ensemble Retrieval\\\small (Vector | Graph | Claim)};
\node (ecr) [below of=retrieval, fill=blue!10] {\textbf{Entropic Claim Resolution}\\\small Entropy-Guided Selection};
\node (synth) [below of=ecr] {Response Synthesis};

\draw[arrow] (query) -- (retrieval);
\draw[arrow] (retrieval) -- (ecr);
\draw[arrow] (ecr) -- node[right, xshift=2pt] {\small epistemic sufficiency} (synth);

\node (hyp) [right=3.6cm of ecr, draw=none, align=left] {
\small\textbf{Inside ECR:}\\
\small Hypothesis space $\mathcal{A}$\\
\small Posterior $P(A \mid X)$\\
\small EER-based claim selection
};

\draw[dashed] (ecr.east) -- (hyp.west);

\end{tikzpicture}
\caption{
System overview: Entropic Claim Resolution (ECR) operates as an inference-time controller between competitive retrieval and answer synthesis. Given a retrieved claim set, ECR sequentially selects evidence to minimize hypothesis entropy and terminates when epistemic sufficiency is reached.
}
\label{fig:architecture}
\end{figure}
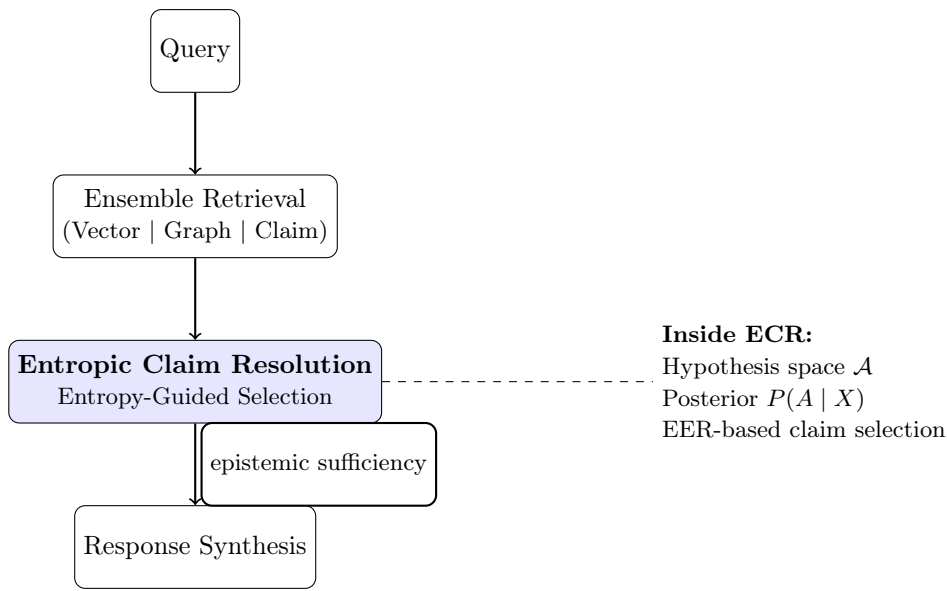

\subsection{HyRAG v3 Ingestion and Index Construction}

\paragraph{Structured and Tabular Data as First-Class Evidence.}
HyRAG v3 natively supports structured and semi-structured tabular data, rather than treating tables as flattened text. During ingestion, the system performs automatic schema inference, including column typing (numeric, categorical, temporal), identifier detection, and time-series normalization. Individual table cells and derived aggregates are materialized as atomic claims with explicit provenance, row identifiers, column metadata, and canonical time keys. Structured aggregation queries are grounded through a text-to-SQL execution path with guarded, read-only execution and validation against real table values. All tabular claims enter the same inference-time evidence pool as textual and graph-derived claims, allowing Entropic Claim Resolution to reason uniformly over mixed structured and unstructured evidence. This design enables precise numeric grounding, temporal filtering, and auditable reasoning not natively supported by graph-enhanced RAG systems that operate over synthesized document summaries.

\paragraph{Vector-Based Retrieval as a Core Substrate.}
HyRAG v3 fully incorporates dense vector retrieval as a primary evidence acquisition mechanism. Raw document chunks, atomic claims, and synthesized summaries are embedded into dedicated vector indices and queried using cosine similarity with optional metadata and identifier filtering. Vector retrieval is used to seed claim pools, initialize hypothesis construction, and ground subsequent structured and graph-based reasoning. Rather than assuming vector similarity implies evidentiary sufficiency, HyRAG v3 subjects all vector-retrieved candidates to inference-time evaluation under Entropic Claim Resolution, allowing relevance-based signals to be retained while preventing overconfidence in semantically similar but non-discriminative evidence.

ECR operates at inference time, but its effectiveness depends on upstream ingestion and indexing that preserve atomicity, provenance, and temporal structure. The implemented HyRAG v3 pipeline (in our reference implementation) performs the following steps.

\paragraph{Auto-adaptive schema inference with feedback calibration.}
An \texttt{AutoAdaptAgent} infers a schema from CSV/Excel/PDF/DataFrame inputs, identifying an ID column, categorical columns, numeric columns, and time-series columns. A subsequent \emph{schema feedback loop} performs a dry-run parse of the first $N$ rows (configurable) and adjusts misclassified columns (e.g., ``numeric'' columns with excessive null-rates), producing a corrected schema used for full ingestion.

\paragraph{Robust parsing with repeated-header detection and temporal normalization.}
The ingestion parser supports multiple formats and implements spreadsheet-specific heuristics, including merging complementary multi-row headers and skipping repeated header rows using an overlap threshold ($\ge 0.70$ token overlap). Time-series columns are normalized via a data-driven time-key parser that recognizes patterns such as years (e.g., 2024), quarters (e.g., 2024Q1), halves (e.g., 2024H2), and trailing windows (e.g., LTM/TTM), and maps them to a canonical order key used for temporal slicing.

\paragraph{EAV SQLite store with safe query execution.}
All ingested records are persisted in an Entity--Attribute--Value SQLite backend (\texttt{GenericStore}). For downstream aggregation queries, the system exposes a text-to-SQL route but enforces a strict \texttt{SELECT}-only guardrail: the SQL executor blocks write operations and limits result sizes.

\paragraph{Embeddings and vector indices with deterministic fallbacks.}
The embedding subsystem is three-tiered: an online embedding API (if available), a local sentence-transformer fallback, and a deterministic hashed-vector fallback for fully offline operation. Vector indices support an optional database backend (LanceDB when installed) and a pure NumPy cosine-similarity backend otherwise; both support ID filtering for category/time constraints.

\paragraph{Atomic claim extraction and claim index.}
During ingestion, the system extracts atomic claims, entities, and lightweight semantic relations $(h,r,t)$ into a dedicated \texttt{ClaimStore}. Claim vectors are embedded and stored in a separate claim vector index to enable claim-first retrieval.

\paragraph{Hierarchical summarization as retrievable nodes.}
To improve global recall, the system clusters embedded row representations using a pure-NumPy $k$-means routine (no external ML dependencies), summarizes each cluster (LLM when available), re-embeds the summaries, and inserts them into the same row-level vector index under a reserved ID prefix. As a result, standard vector retrieval can surface both raw rows and higher-level cluster summaries.
Cluster summaries are stored as first-class retrievable nodes and compete directly with raw rows during vector retrieval.

ECR is exclusively activated on the analytical \texttt{CSGR\_PLUS} route selected by the upstream query router, and is bypassed for \texttt{LOOKUP}, \texttt{RELATIONAL}, \texttt{SEMANTIC}, \texttt{TOOL}, and \texttt{SQL} routes.

External tools are treated as deterministic operators outside the entropy-driven reasoning loop (the LLM only formats a JSON tool call when available, with an offline numeric-statistics fast-path), i.e., excluded from ECR's epistemic modeling rather than treated as competing uncertainty-reduction actions.

To bound computational overhead, ECR is invoked dynamically strictly when the retrieved configuration exhibits high epistemic uncertainty. The trigger conditions are natively integrated via three heuristics:
\begin{enumerate}
    \item \textbf{High Claim Volume:} The retriever fetches heavily saturated candidate spaces ($>15$ claims).
    \item \textbf{Syntactic Ambiguity:} Detection of uncertainty keywords within the active query (e.g., ``uncertain'', ``conflicting'', ``disagree'', ``multiple'', ``various'').
    \item \textbf{Confidence Variance Constraint:} The variance in micro-level claim confidence $\sigma^2$ across the $k$ retrieved claims exceeds an empirical threshold of $0.15$ (with confidence actively tied to tracking topological support-contradiction metrics within the underlying datastore).
\end{enumerate}

To evaluate ECR in a high-performance setting, we implement it as a standalone and modular resolution engine within a production-grade Context-Seeded Graph Retrieval (CSGR++) architecture. While ECR is algorithmically orthogonal to any specific retriever, CSGR++ serves as a rigorous experimental testbed that preserves atomicity, provenance, and multi-strategy retrieval signals. Knowledge is extracted and stored as atomic semantic claims in an Entity-Attribute-Value (EAV) backend, accompanied by separate vector indices for raw rows and claims.

Within this testbed, the baseline multi-strategy \emph{EnsembleRetriever} combines dense similarity search, structural graph expansion, and semantic claim matching using Reciprocal Rank Fusion. ECR cleanly intercepts the pipeline immediately after candidate generation, acting as an isolated inference-time uncertainty resolution stage that outputs either a dominant hypothesis or a calibrated set of alternatives for downstream synthesis.

\subsection{CSGR++ Backbone Architecture}
While ECR is algorithmically orthogonal to a particular retrieval stack, we implement and evaluate it inside a production-grade pipeline (CSGR++) that is explicitly claim-centric.

\paragraph{Atomic claim store with semantic relations.}
CSGR++ stores extracted claims in a SQLite-backed \emph{ClaimStore} with fields for (i) claim text, (ii) entity mentions, (iii) time keys / order keys for temporal slicing, and (iv) dynamically updated confidence signals. In addition, a lightweight semantic relation table stores tuples $(h, r, t)$ extracted during claim extraction (e.g., \textsc{Acquires}, \textsc{Impacts}, \textsc{CausedBy}), enabling entity-based expansion during retrieval.

\paragraph{Temporal intelligence.}
Queries are parsed for explicit time constraints (e.g., ``in 2024'', ``2024Q1'', ``last 3 quarters'', ``since 2022'') and converted into an order-key interval $(\tau_{min},\tau_{max})$. Claim retrieval can then apply a hard filter over the claim IDs inside the selected time window.

\paragraph{Competitive ensemble retrieval and Reciprocal Rank Fusion (RRF).}
The retriever runs multiple strategies (vector retrieval over rows, vector retrieval over claims, and graph/category traversal) and fuses the per-strategy rankings via Reciprocal Rank Fusion (RRF). For an item $d$ and ranking lists $\{L_j\}_{j=1}^m$ with ranks $r_j(d)\in\{1,2,\ldots\}$, the fused score is
\begin{equation}
\label{eq:rrf}
\text{RRF}(d) = \sum_{j=1}^{m} \frac{1}{k + r_j(d)},
\end{equation}
where $k$ is a dampening constant (we use $k=60$ in code).

\paragraph{Competitive strategy scoring (selection, not only fusion).}
In addition to fusing rankings, the retriever scores each strategy to identify a ``best'' strategy for the query. The implemented scoring combines (i) average similarity score, (ii) a diversity proxy based on unique source items, and (iii) average claim confidence (when applicable) via a weighted sum.

Beyond rank fusion, this strategy scoring identifies the dominant evidence view for a query, enabling adaptive retrieval-path selection rather than blindly trusting an ensemble.

\paragraph{Relation-based expansion for multi-hop analytical queries.}
For analytical (CSGR++) queries, the system extracts frequent entities from initially retrieved claims, then expands the evidence set by retrieving related claims via the relation table (one-hop expansion), discounting confidence slightly for expanded claims.

\paragraph{Dynamic confidence micro-learning.}
Claims maintain support and contradiction counters. When verification indicates a claim was supported or contradicted, the system updates its confidence with a bounded, asymmetric rule:
\begin{equation}
\label{eq:dynconf}
\text{conf}_{new}(c) = \text{clip}_{[0,1]}\Big(\text{conf}_{base}(c) + 0.15\,\log(1+S(c)) - 0.25\,C(c)\Big).
\end{equation}
This produces an online ``micro-learning'' effect: frequently supported claims become easier to trust, while contradicted claims are rapidly down-weighted.

Because claim confidence is updated online and directly affects future $\widehat{\mathrm{EER}}(c)$ scores (Eq.~\ref{eq:eer_impl}), ECR exhibits lightweight inference-time learning behavior across queries.

\paragraph{Trust modes (graded verification).}
The query router classifies user intent into trust modes (\emph{strict} for regulatory or numerical precision, \emph{balanced}, and \emph{exploratory}), which modulate verification aggressiveness and synthesis style.

\paragraph{ReverseVerifier: deterministic numeric grounding + claim-aware checking.}
Beyond probabilistic resolution, CSGR++ applies a three-layer ReverseVerifier: (i) a deterministic numeric grounding pass that extracts all numeric tokens in a draft answer and checks verbatim presence in retrieved evidence, (ii) LLM-based claim-by-claim judgement with both supporting and counter-evidence retrieval, and (iii) a combined score where numeric failures cap the maximum achievable verification score. Numeric grounding is enforced as a hard constraint: a single unsupported numeric token caps downstream verification scores.

Table~\ref{tab:components} summarizes the major subsystems of the full HyRAG v3 and CSGR++ pipeline and their respective roles, providing a compact overview of how ECR integrates into the surrounding retrieval, verification, and synthesis infrastructure.

\begin{table}[t]
\centering
\caption{Key subsystems implemented in our system that support ECR and the full end-to-end pipeline.}
\label{tab:components}
\begin{tabular}{p{0.42\linewidth}p{0.53\linewidth}}
\toprule
Subsystem & Role in the pipeline \\
\midrule
AutoAdaptAgent + SchemaFeedbackLoop & Schema inference with dry-run calibration \\
GenericStore (EAV SQLite) & Item/attribute persistence; safe \texttt{SELECT}-only SQL execution \\
EmbeddingProvider + VectorIndex & 3-tier embeddings; LanceDB/NumPy backends; ID-filtered cosine search \\
ClaimExtractor + ClaimStore & Atomic claims + relations + temporal keys + dynamic confidence \\
EnsembleRetriever & Competitive retrieval + RRF fusion (Eq.~\ref{eq:rrf}) \\
EntropicClaimResolver & ECR loop: entropy, EER selection (Eq.~\ref{eq:eer_impl}) \\
StructuredSynthesizer & Structured analytical brief with evidence bullets \\
ReverseVerifier & Numeric grounding + claim-aware verification and score capping \\
RAGAnswerer & Multi-hop, HyDE, text-to-SQL grounding, CRAG self-correction, citations \\
\bottomrule
\end{tabular}
\end{table}

\subsection{Supporting RAG Components}
Outside the CSGR++ analytical route, the implementation includes a general-purpose RAG engine that packages standard, widely used RAG mechanisms behind a single \texttt{answer()} interface. These components are supporting infrastructure and are orthogonal to ECR.

The system also supports generator-based streaming responses (via \texttt{answer\_stream} entry-points), which is orthogonal to ECR and not evaluated in this work.\footnote{All major components admit deterministic fallbacks when LLMs are unavailable (e.g., hashed embeddings and heuristic claim extraction), though answer quality may degrade.}

\paragraph{Multi-hop retrieve--reason--retrieve.}
The engine iteratively retrieves candidates and, when online, generates a follow-up query conditioned on current evidence, stopping early when additional hops yield no new items.

\paragraph{HyDE query embedding.}
To improve recall under distribution shift between user queries and row-shaped embeddings, the engine optionally generates a short hypothetical ``answer row'' and embeds that text (HyDE) to drive vector search.

\paragraph{Cross-encoder reranking and calibrated abstention.}
Candidates are reranked either by cosine similarity (offline) or by an LLM ``cross-encoder'' that outputs a ranking and confidence. A calibrated confidence score combines the number of retrieved results, the top similarity score, the reranker confidence, and a query-complexity penalty; the system abstains when the calibrated score is low.

\paragraph{Text-to-SQL with value grounding.}
For aggregation queries, the engine routes to text-to-SQL and applies a second grounding pass that validates every generated string literal against real categorical values in the database; when an unknown literal is detected, it is rewritten to the closest fuzzy match when possible.

\paragraph{CRAG self-correction with schema evolution signals.}
When reverse verification returns \texttt{fail} or \texttt{weak}, the engine performs up to two correction attempts by rewriting the query to target the verification gap. Each failure can be recorded by a schema evolution tracker that increments per-column failure counts and can request LLM-based reclassification suggestions once a threshold is exceeded.
Schema evolution signals persist across queries, enabling long-term self-correction.

\section{Experimental Design \& Evaluation}

While Section~4 outlines the deployment of ECR within a full-scale production
architecture, evaluating the algorithm end-to-end immediately introduces
confounding variables from upstream retrieval recall and downstream LLM
generation quality. To rigorously validate the decision-theoretic properties
established in Section~3, our evaluation strategy proceeds in two phases. First,
we strictly isolate the mathematical behavior of the entropy-driven claim
selection policy using a controlled, claims-only harness
(Sections~5.1--5.3). Second, we reintegrate ECR into an end-to-end reasoning
pipeline to evaluate its impact under realistic multi-hop and contradiction-heavy
settings (Section~5.4).

\subsection{Controlled claims-only harness}
Our ``claims-only'' harness fixes the dataset, query set, retrieval configuration, candidate claim pool, and Bayesian entropy model; only the claim-selection policy differs.
This allows a clean measurement of whether a policy is actually minimizing epistemic uncertainty as defined by Eq.~\ref{eq:entropy}.

\paragraph{Dataset and cases.}
We use a small, multi-table business dataset of six CSV tables (\texttt{sales}, \texttt{customers}, \texttt{expenses}, \texttt{inventory}, \texttt{hr}, \texttt{marketing}) and 80 templated evaluation queries spanning single-table lookups and cross-table comparisons.

\paragraph{Hypotheses and initial entropy.}
For each query, the harness constructs $|\mathcal{A}|=3$ mutually exclusive answer hypotheses, yielding an initial entropy of $H_0=\log_2 3 \approx 1.585$ bits.

\paragraph{Candidate claims and policies.}
For each case, we retrieve the same top-20 candidate claims (high-recall candidate generation).
We then compare three policies: (i) \textbf{Retrieval-only}, which takes the top-15 claims by retrieval score under a fixed budget; (ii) \textbf{ECR}, which sequentially selects the next claim by expected entropy reduction and stops when $H\le\epsilon$ with $\epsilon=0.3$ bits (capped at 10 iterations); and (iii) \textbf{Random control}, which samples claims uniformly without replacement from the same candidate pool, matching ECR's realized claim budget of 5 claims.

\paragraph{Entropy-aligned metrics.}
We report (i) final entropy, (ii) entropy drop per evaluated claim, (iii) claims-to-collapse (first step reaching $H\le\epsilon$, else budget$+1$), (iv) effective hypotheses ($2^H$), and (v) entropy trace variance.
We additionally report two diversity-oriented diagnostics (claim redundancy and source entropy) to illustrate that diversity alone is not equivalent to epistemic resolution.
Finally, we report \emph{hypothesis-conditioned redundancy} (HypCondRed.), which computes redundancy within claim groups attributed to the same answer hypothesis (rather than across the full mixed set).

\subsection{Main results (seed=7, 80 cases)}
Table~\ref{tab:claims_only_results} summarizes mean $\pm$ std across cases.
ECR reliably reaches epistemic sufficiency using 5 claims, driving $H$ below $\epsilon$; retrieval-only does not reduce entropy under the same posterior model; random improves modestly but typically does not collapse.

\begin{table}[t]
\centering
\small
\resizebox{\linewidth}{!}{%
{\setlength{\tabcolsep}{3pt}\renewcommand{\arraystretch}{0.95}%
\begin{tabular}{@{}lrrrrrrrr@{}}
\toprule
\textbf{Policy} & \textbf{Claims} & \textbf{$H_{final}$} & \textbf{$\Delta H$/claim} & \textbf{Collapse} & \textbf{$2^{H_{final}}$} & \textbf{Redund.} & \textbf{HypCondRed.} & \textbf{SrcEnt}\\
\midrule
Retrieval-only & $15.0\pm0.0$ & $1.585\pm0.000$ & $0.0000\pm0.0000$ & $16.0\pm0.0$ & $3.000\pm0.000$ & $0.684\pm0.119$ & $0.662\pm0.110$ & $0.342\pm0.510$\\
ECR & $5.0\pm0.0$ & $0.2129\pm0.0000$ & $0.2744\pm0.0000$ & $5.0\pm0.0$ & $1.159\pm0.000$ & $0.672\pm0.125$ & $0.672\pm0.125$ & $0.276\pm0.443$\\
Random & $5.0\pm0.0$ & $1.243\pm0.289$ & $0.0684\pm0.0577$ & $6.0\pm0.0$ & $2.411\pm0.437$ & $0.658\pm0.118$ & $0.653\pm0.123$ & $0.354\pm0.527$\\
\bottomrule
\end{tabular}}}
\caption{Claims-only evaluation (80 cases, seed=7). ``Claims'' is the number of evaluated claims. $H_{final}$ is final answer-hypothesis entropy in bits. ``Collapse'' is claims-to-collapse (first step where $H\le\epsilon=0.3$; else budget$+1$). ``Redund.'' is claim redundancy, ``HypCondRed.'' is hypothesis-conditioned claim redundancy, and ``SrcEnt'' is source entropy (diagnostic diversity metrics).}
\label{tab:claims_only_results}
\end{table}

Across these runs, claim-coverage is identical across policies (0.6375 on average), reflecting that this harness is designed to stress epistemic resolution rather than maximize overlap with a small set of expected claim snippets.

\subsection{Robustness across random seeds (seeds 0--4)}
To ensure the random-control comparison is not a single-seed artifact, we rerun the claims-only harness for five random seeds (0--4), reusing the same frozen dataset, query set, candidate claims, and posterior model.
Retrieval-only and ECR are deterministic under this setup, while the random baseline varies by construction.

Table~\ref{tab:multiseed_results} confirms the stability of ECR across multiple seeds. Furthermore, Figure~\ref{fig:convergence} illustrates the schematic entropy trajectories of these competing policies, highlighting how rapidly ECR drives the hypothesis space below the $\epsilon$ threshold compared to relevance-only baselines.

\begin{table}[t]
\centering
\small
\resizebox{\linewidth}{!}{%
{\setlength{\tabcolsep}{4pt}\renewcommand{\arraystretch}{0.95}%
\begin{tabular}{@{}lrrrrrr@{}}
\toprule
\textbf{Policy} & \textbf{$H_{final}$} & \textbf{$\Delta H$/claim} & \textbf{Collapse} & \textbf{$2^{H_{final}}$} & \textbf{TraceVar} & \textbf{HypCondRed.}\\
\midrule
Retrieval-only & $1.585\pm0.000$ & $0.0000\pm0.0000$ & $16.00\pm0.00$ & $3.000\pm0.000$ & $0.002987\pm0.000000$ & $0.6619\pm0.0000$\\
ECR & $0.2129\pm0.0000$ & $0.2744\pm0.0000$ & $5.00\pm0.00$ & $1.159\pm0.000$ & $0.262859\pm0.000000$ & $0.6719\pm0.0000$\\
Random & $1.2628\pm0.0265$ & $0.0644\pm0.0053$ & $5.995\pm0.006$ & $2.436\pm0.044$ & $0.03210\pm0.00374$ & $0.6401\pm0.0075$\\
\bottomrule
\end{tabular}}}
\caption{Multi-seed robustness (seeds 0--4): mean $\pm$ std over seeds of the seed-level mean metrics. Only the random baseline changes across seeds in this frozen setup. ``HypCondRed.'' is hypothesis-conditioned claim redundancy.}
\label{tab:multiseed_results}
\end{table}

\begin{figure}[ht]
\centering
\resizebox{\linewidth}{!}{%
\begin{tikzpicture}[x=1.0cm,y=3.6cm]
\draw[->] (0,0) -- (10.5,0) node[below right] {Evidence Claims Actively Evaluated ($t$)};
\draw[->] (0,0) -- (0,1.9) node[above left, align=left] {Hypothesis Entropy\\$H(A\mid X_{eval})$ (bits)};

\foreach \x in {0,1,...,10} {
    \draw[gray!25] (\x,0) -- (\x,1.8);
    \draw (\x,0) -- (\x,-0.02);
    \node[below] at (\x,-0.02) {\scriptsize \x};
}
\foreach \y/\ylab in {0/0.0,0.3/0.3,0.6/0.6,0.9/0.9,1.2/1.2,1.5/1.5,1.8/1.8} {
    \draw[gray!25] (0,\y) -- (10,\y);
    \draw (0,\y) -- (-0.06,\y);
    \node[left] at (-0.06,\y) {\scriptsize \ylab};
}

\draw[black, dotted, thick] (0,0.30) -- (10,0.30);
\node[above] at (5,0.30) {\scriptsize Epistemic sufficiency ($\epsilon=0.3$)};

\draw[blue, thick] plot coordinates {(0,1.585) (2,0.90) (4,0.35) (5,0.213)};
\foreach \x/\y in {0/1.585,2/0.90,4/0.35,5/0.213} {
    \fill[blue] (\x,\y) circle (1.6pt);
}

\draw[red, thick, dashed] plot coordinates {(0,1.585) (2,1.45) (4,1.30) (5,1.26)};
\foreach \x/\y in {0/1.585,2/1.45,4/1.30,5/1.26} {
    \fill[red] (\x,\y) rectangle +(0.08,0.04);
}

\begin{scope}[shift={(6.0,1.62)}]
    \draw[blue, thick] (0,0.14) -- (0.9,0.14);
    \fill[blue] (0.45,0.14) circle (1.6pt);
    \node[anchor=west] at (1.0,0.14) {\scriptsize ECR (schematic)};

    \draw[red, thick, dashed] (0,0) -- (0.9,0);
    \fill[red] (0.43,-0.02) rectangle +(0.08,0.04);
    \node[anchor=west] at (1.0,0) {\scriptsize Random/retrieval (schematic)};
\end{scope}
\end{tikzpicture}}%
\caption{Schematic entropy trajectories consistent with the measured endpoints: ECR reaches $H\le\epsilon$ quickly, whereas relevance-only and random baselines typically remain above $\epsilon$ at matched claim budgets.}
\label{fig:convergence}
\end{figure}
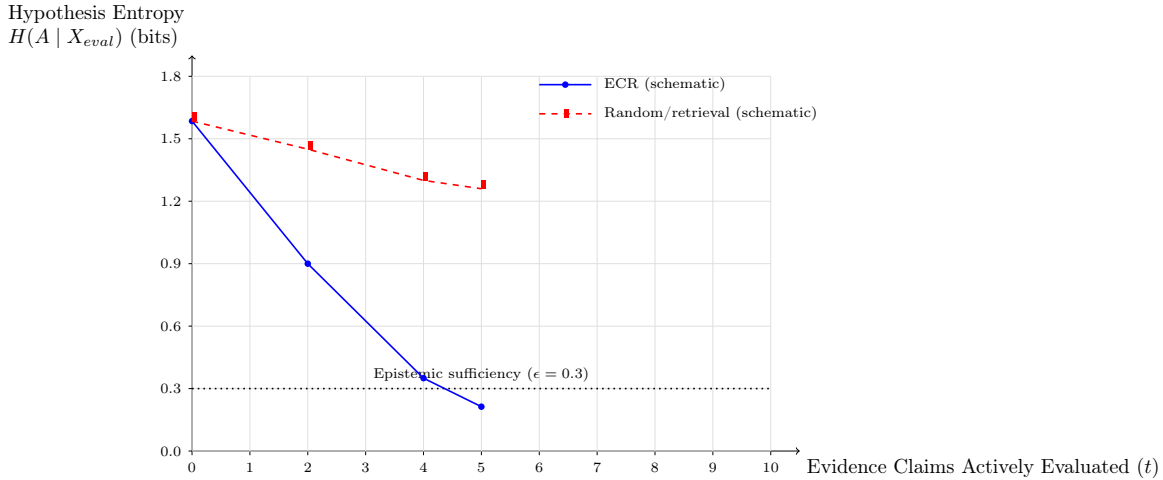

\subsection{End-to-End Evaluation on a Standard Multi-Hop QA Benchmark}
In contrast to the preceding controlled, claims-only experiments, this evaluation reintegrates a live large language model into the inference loop, exercising ECR as an online evidence-selection controller during end-to-end RAG generation.
As an additional experiment, to evaluate whether entropy-guided evidence selection improves downstream answer quality, we conduct an end-to-end evaluation on a HotpotQA-style multi-hop QA benchmark. All methods share the same retriever, language model, candidate evidence pool, and decoding parameters; the only variable is the inference-time claim selection policy. We evaluate three policies: (i) a relevance-based baseline RAG policy, (ii) a random selection control matched to the same average claim budget, and (iii) ECR, which applies entropy-guided selection with stopping. We report exact match (EM), token F1, and an evidence faithfulness proxy based on answer-token coverage, alongside the average number of claims used. Because HotpotQA exhibits substantially higher linguistic variance and more complex multi-hop dependencies than highly structured tabular datasets, the ECR algorithm naturally evaluates a larger number of claims before the hypothesis entropy collapses below $\epsilon$.

\begin{table}[ht]
\centering
\small
\begin{tabular}{lrrrr}
\toprule
\textbf{Method} & \textbf{Avg. Claims Used} & \textbf{Exact Match (EM) $\uparrow$} & \textbf{Token F1 $\uparrow$} & \textbf{Evidence Faithfulness $\uparrow$}\\
\midrule
Baseline RAG & 19.87 & 0.313 & 0.459 & 0.639\\
Random Control & 19.87 & 0.207 & 0.307 & 0.427\\
ECR (ours) & 19.68 & 0.297 & 0.450 & 0.626\\
\bottomrule
\end{tabular}
\caption{End-to-End Evaluation on HotpotQA-Style Multi-Hop QA (300 Questions). All methods use the same retriever and language model; only the inference-time evidence selection policy differs. ECR substantially outperforms random selection while maintaining performance comparable to a strong relevance-based baseline.}
\label{tab:hotpot_e2e_results}
\end{table}

Table~\ref{tab:hotpot_e2e_results} shows that ECR substantially outperforms random selection across all reported metrics, confirming that entropy-guided evidence selection is consistently more effective than unguided or diversity-only strategies. Relative to a strong relevance-based baseline, ECR remains within a small margin on EM and F1, indicating that enforcing epistemic control does not significantly degrade answer accuracy on standard benchmarks.

It is important to note that HotpotQA is a largely factual and relevance-oriented benchmark with predominantly singular ground truths. As such, it does not natively stress-test contradictory evidence or fundamentally ambiguous queries, which are precisely the regimes ECR is designed to address. Achieving near parity on such a saturated benchmark while enforcing strict inference-time epistemic constraints demonstrates that ECR integrates robust uncertainty control without reliance on benchmark-specific tuning. Future evaluations will focus on conflict-heavy or ambiguity-oriented benchmarks where relevance-driven retrieval is known to exhibit epistemic collapse.

\paragraph{Robustness to Noisy Evidence.}
To isolate a regime that is closer to real deployments, where retrieved evidence may include irrelevant or even contradictory content, we perform a controlled ablation on the same HotpotQA evaluation set and pipeline as above, injecting noise \emph{after retrieval and before evidence selection}\footnote{Because noise is injected by replacing a fraction of candidate claims, this protocol may remove gold evidence for some queries. Consequently, Exact Match under heavy corruption reflects robustness to partial evidence loss rather than distractor filtering.}. For each query, we take the retrieved candidate claim set and replace 40\% of candidates with claims sampled from a noise pool constructed from unrelated documents (keeping the retriever, LLM, prompts, decoding, and ECR selection logic unchanged). Table~\ref{tab:hotpot_noise_robustness} reports Exact Match (EM) and Evidence Faithfulness for baseline relevance-based RAG and ECR under no noise versus 40\% noise.

Under this corruption regime, Exact Match necessarily degrades for both systems, as replacing a fraction of candidate claims can remove ground-truth evidence from the pool. Notably, ECR exhibits predictable degradation to the relevance-based baseline without amplifying noise-induced errors, despite enforcing strict inference-time stopping and evaluating fewer claims. This result indicates that entropy-guided evidence selection remains well-behaved under partial evidence loss, avoiding overconfident hallucination or unstable collapse when the available evidence becomes incomplete or unreliable.

We emphasize that this ablation evaluates robustness to evidence corruption (i.e., partial removal of valid claims), rather than distractor accumulation, which isolates a complementary but distinct failure mode.

\begin{table}[t]
\centering
\small
\begin{tabular}{lrrrr}
\toprule
\textbf{Method} & \textbf{EM (No Noise)} & \textbf{Faith (No Noise)} & \textbf{EM (40\% Noise)} & \textbf{Faith (40\% Noise)}\\
\midrule
Baseline RAG & 0.323 & 0.660 & 0.167 & 0.345\\
ECR (ours) & 0.307 & 0.657 & 0.163 & 0.331\\
\bottomrule
\end{tabular}
\caption{Robustness ablation on HotpotQA-style evaluation (300 questions) with noise injected \emph{after retrieval and before evidence selection}. ``40\% Noise'' replaces 40\% of retrieved candidate claims with unrelated (potentially contradictory) claims sampled from a noise pool. Only baseline relevance-based RAG and ECR are evaluated; the retriever, LLM, prompts, and decoding are unchanged. (Performance is bounded above when ground-truth evidence is removed.)}
\label{tab:hotpot_noise_robustness}
\end{table}

\paragraph{Offline Robustness Under Structured Contradiction.}
Standard QA benchmarks predominantly evaluate answer accuracy under relatively clean evidence conditions. To stress-test the epistemic-control mechanism itself---independently of LLM semantics---we run a fully offline, deterministic contradiction-injection ablation on the same 300-question HotpotQA-style set and retrieval pipeline. For each query, we take the retrieved candidate claim pool and inject paired, explicit contradiction twins into the candidate set at rate $\alpha\in\{0.0,0.3,0.5\}$ after retrieval and before evidence selection. In offline mode, hypothesis initialization uses deterministic hashed embeddings and claim verification uses a deterministic provenance proxy; this isolates controller behavior from verifier quality.

We report (i) \textbf{Ambiguity Exposure}---whether the run ends with $H>\epsilon$ or an unresolved explicit contradiction pair---and (ii) \textbf{Overconfident Error}---cases where the system outputs a dominant hypothesis despite being wrong (a proxy for epistemic collapse). Table~\ref{tab:offline_contradiction_injection} shows a sharp regime shift: baseline relevance-based RAG remains pathologically overconfident and flat across $\alpha$, while ECR transitions from fast epistemic sufficiency in the clean regime ($\alpha=0.0$) to principled non-convergence under contradiction ($\alpha\ge 0.3$). At $\alpha\ge 0.3$, ambiguity emerges deterministically for every query and termination is entirely explained by unresolved conflict rather than heuristic budget limits. This extreme ambiguity rate is expected: once an explicit contradiction pair is present in the evaluated evidence, epistemic coherence is unattainable by definition. Likewise, entropy remains high because ECR is not an entropy minimizer ``at all costs''; it is a coherence-constrained entropy controller.

\begin{table}[t]
\centering
\small
\begin{tabular}{llrrrrl}
\toprule
\textbf{Method} & $\boldsymbol{\alpha}$ & \textbf{EM} & \textbf{OverconfErr} & \textbf{AmbExp} & \textbf{Mean $H$} & \textbf{Stop Reason}\\
\midrule
Baseline RAG & 0.0 & 0.0067 & 0.9933 & 0.0000 & -- & fixed\_budget (300/300)\\
Baseline RAG & 0.3 & 0.0067 & 0.9933 & 0.0000 & -- & fixed\_budget (300/300)\\
Baseline RAG & 0.5 & 0.0067 & 0.9933 & 0.0000 & -- & fixed\_budget (300/300)\\
\midrule
ECR (ours) & 0.0 & 0.0000 & 0.9900 & 0.0100 & 0.226 & epistemic\_sufficiency (297/300)\\
ECR (ours) & 0.3 & 0.0067 & 0.0000 & 1.0000 & 1.496 & unresolved\_conflict (300/300)\\
ECR (ours) & 0.5 & 0.0067 & 0.0000 & 1.0000 & 1.458 & unresolved\_conflict (300/300)\\
\bottomrule
\end{tabular}
\caption{Offline contradiction-injection ablation (300 questions). Paired contradictions are injected into the candidate claim pool at rate $\alpha$ after retrieval and before evidence selection. EM is reported only as a sanity anchor under a deterministic offline answerer; the key signals are Ambiguity Exposure and Overconfident Error (epistemic collapse). ECR exhibits a phase transition from epistemic sufficiency to principled non-convergence as contradictions accumulate, while baseline RAG remains uniformly overconfident. Counts indicate number of runs terminating for each reason.}
\label{tab:offline_contradiction_injection}
\end{table}

Exploring complementary ambiguity-focused benchmarks and distractor accumulation regimes remains an important direction for future evaluation.

\section{Conclusion}

\subsection*{Summary}

Entropic Claim Resolution introduces a principled inference-time perspective on Retrieval-Augmented Generation, reframing evidence selection as a process of epistemic uncertainty reduction rather than relevance maximization. By directly optimizing Expected Entropy Reduction over atomic claims, ECR provides a mathematically grounded mechanism for determining both which evidence to evaluate next and when sufficient evidence has been accumulated to justify synthesis.

Empirically, we show that this entropy-driven framework reliably collapses hypothesis uncertainty in controlled claim-level settings and substantially outperforms random evidence selection in end-to-end multi-hop question answering, while maintaining performance comparable to strong relevance-based baselines. These results highlight a fundamental distinction between optimizing for raw answer accuracy and enforcing principled epistemic control during inference.

In a fully offline contradiction-injection stress test, ECR exhibits a sharp transition from epistemic sufficiency to principled non-convergence as structured conflict accumulates: entropy ceases to collapse, evidence exploration increases, and termination is explained by unresolved inconsistency rather than heuristic budgets.

Unlike retrieval architectures designed primarily for long-form unstructured documents, HyRAG v3 explicitly models structured tabular data with row-level grounding, enabling ECR to enforce numeric correctness and temporal consistency during inference.

Beyond benchmark performance, the ECR framework offers clear advantages for real-world and enterprise deployments. In high-stakes domains such as medicine, law, and finance, confidently synthesizing a single answer from conflicting or incomplete evidence can be costly or harmful. By providing a mathematically grounded mechanism to expose unresolved ambiguity when epistemic sufficiency cannot be reached, ECR functions as a principled constraint against unhedged generation. Furthermore, the ability to dynamically halt evidence accumulation once $H \le \epsilon$ is satisfied mitigates unnecessary computational overhead, reducing latency and cost associated with processing large, redundant context windows. This positions ECR as a resource-efficient inference-time control mechanism for scalable and risk-aware AI reasoning.

\subsection*{Limitations and Future Work}

We conclude by outlining key limitations of the current framework and highlighting promising directions for future research.

\paragraph{Hypothesis space coverage.}
A primary limitation of the current framework is its reliance on the initial hypothesis generation stage. Entropic Claim Resolution operates over an explicitly constructed hypothesis set and therefore inherits a bounded-coverage assumption; if the true answer is entirely absent from this space, the system may converge confidently to an incorrect explanation. In practice, this limitation can be mitigated by regenerating hypotheses when entropy fails to decrease or when accumulated evidence weakly supports all candidates. Future work will explore dynamic mid-loop hypothesis extension, soft hypothesis assignments, richer likelihood models, and tighter integration with learned retrievers to further strengthen entropy-guided reasoning under uncertainty.
Importantly, ECR’s refusal to converge under explicit contradiction is a deliberate design choice rather than a limitation: when the evaluated evidence is epistemically incoherent, the framework exposes ambiguity instead of forcing posterior collapse. This behavior preserves epistemic correctness but may yield non-decisive outputs in genuinely inconsistent corpora.

An orthogonal robustness regime involves distractor accumulation without evidence removal, which we leave to future investigation.

Finally, while this work evaluates Entropic Claim Resolution specifically within Retrieval-Augmented Generation, the underlying methodology naturally extends to agentic and autonomous contexts. Our approach suggests a perspective where agent actions (such as executing tools, querying external APIs, or taking exploratory steps) can be modeled dynamically as entropy-minimizing decisions evaluated under a rigorous Expected Entropy Reduction criterion. This aligns with recent advancements in autonomous cognitive control, including topology-aware routing \cite{digioia2026cascade} and dynamic temporal pacing \cite{digioia2026learning}, providing a formal alternative to standard prompt-driven or heuristic action-selection policies. We view the integration of decision-theoretic primitives into continuous agentic feedback loops as a compelling frontier for building robust and mathematically grounded autonomous systems.

\clearpage
\FloatBarrier

\clearpage

\section*{Appendix}
\addcontentsline{toc}{section}{Appendix}

\appendix

\setcounter{figure}{0}
\setcounter{table}{0}
\renewcommand{\thefigure}{A.\arabic{figure}}
\renewcommand{\thetable}{A.\arabic{table}}

\section{\texorpdfstring{$\lambda$-Sweep Robustness}{Lambda-Sweep Robustness}}
To test whether coherence-aware behavior requires fragile tuning, we sweep the coherence bonus weight $\lambda$ in ECR's evidence selection policy over $\{0, 0.01, 0.025, 0.05, 0.1\}$ while keeping the offline protocol, budgets, and contradiction injection rates $\alpha\in\{0.0,0.3,0.5\}$ fixed. We observe that $\lambda=0$ behaves as entropy-only control and can converge to a dominant hypothesis even under contradiction injection, whereas any tested non-zero $\lambda$ yields the same coherence-aware regime in which explicit contradictions are surfaced and prevent epistemic collapse; consequently, behavior saturates across all tested $\lambda>0$. We set $\lambda=0.05$ as the default. As shown in Table~\ref{tab:lambda_sweep_summary}, we observe a sharp transition between entropy-only control ($\lambda=0$) and coherence-aware control ($\lambda>0$), with behavior saturating for all tested non-zero values. This indicates that ECR does not require fine-grained hyperparameter tuning to surface epistemic inconsistency.

\begin{figure}[H]
\centering
\begin{tikzpicture}[x=1cm,y=3.5cm]

\draw[->] (0,0) -- (10.6,0) node[below right] {$\lambda_{\mathrm{conflict}}$};
\draw[->] (0,0) -- (0,1.15) node[above left] {Ambiguity Exposure};

\foreach \x/\xl in {0/0,1/0.01,2.5/0.025,5/0.05,10/0.10} {
    \draw[gray!20] (\x,0) -- (\x,1.05);
    \draw (\x,0) -- (\x,-0.025);
    \node[below] at (\x,-0.025) {\scriptsize \xl};
}
\foreach \y in {0,1} {
    \draw[gray!20] (0,\y) -- (10,\y);
    \draw (0,\y) -- (-0.08,\y);
    \node[left] at (-0.08,\y) {\scriptsize \y};
}

\draw[blue, thick] (0,0) -- (10,0);
\foreach \x in {0,1,2.5,5,10} {
    \fill[blue] (\x,0) circle (1.5pt);
}

\draw[orange, thick, dashed] (0,0.015) -- (10,0.015);
\foreach \x in {0,1,2.5,5,10} {
    \fill[orange] (\x,0.015) circle (1.5pt);
}

\draw[green!70!black, thick] (0,0) -- (0.8,0);
\draw[green!70!black, thick] (0.8,0) -- (0.8,1);
\draw[green!70!black, thick] (0.8,1) -- (10,1);

\fill[green!70!black] (0,0) circle (1.5pt);
\foreach \x in {1,2.5,5,10} {
    \fill[green!70!black] (\x,1) circle (1.5pt);
}

\begin{scope}[shift={(6.0,0.8)}]
    \draw[blue, thick] (0,0.18) -- (0.9,0.18);
    \fill[blue] (0.45,0.18) circle (1.5pt);
    \node[right] at (1.0,0.18) {\scriptsize $\alpha = 0.0$};

    \draw[orange, thick, dashed] (0,0.09) -- (0.9,0.09);
    \fill[orange] (0.45,0.09) circle (1.5pt);
    \node[right] at (1.0,0.09) {\scriptsize $\alpha = 0.3$};

    \draw[green!70!black, thick] (0,0.00) -- (0.9,0);
    \fill[green!70!black] (0.45,0) circle (1.5pt);
    \node[right] at (1.0,0.00) {\scriptsize $\alpha = 0.5$};
\end{scope}

\end{tikzpicture}
\caption{
Ambiguity exposure as a function of the coherence weight $\lambda_{\mathrm{conflict}}$
under structured contradiction injection.
Empirically, ambiguity exposure exhibits a sharp phase transition:
for $\alpha=0.5$, exposure jumps from 0 to 1 for any tested $\lambda>0$,
while remaining 0 for $\alpha \le 0.3$ across all tested settings.
}
\label{fig:lambda_sweep_ambiguity}
\end{figure}
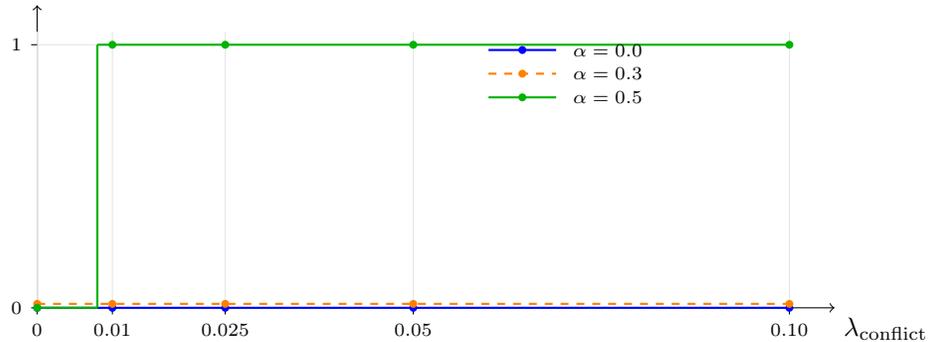

\begin{table}[H]
\centering
\small
\begin{tabular}{lrrrrrr}
\toprule
$\boldsymbol{\lambda}$ & \multicolumn{2}{c}{$\boldsymbol{\alpha=0.0}$} & \multicolumn{2}{c}{$\boldsymbol{\alpha=0.3}$} & \multicolumn{2}{c}{$\boldsymbol{\alpha=0.5}$}\\
 & \textbf{MeanClaims} & \textbf{Mean $H$} & \textbf{MeanClaims} & \textbf{Mean $H$} & \textbf{MeanClaims} & \textbf{Mean $H$}\\
\midrule
0.00  & 5.04 & 0.226 & 5.06 & 0.226 & 5.08 & 0.226\\
0.01  & 5.04 & 0.226 & 25.83 & 1.496 & 29.81 & 1.458\\
0.025 & 5.04 & 0.226 & 25.83 & 1.496 & 29.81 & 1.458\\
0.05  & 5.04 & 0.226 & 25.83 & 1.496 & 29.81 & 1.458\\
0.10  & 5.04 & 0.226 & 25.83 & 1.496 & 29.81 & 1.458\\
\bottomrule
\end{tabular}
\caption{$\lambda$-sweep summary statistics (offline, deterministic). Values are aggregated over 300 questions.}
\label{tab:lambda_sweep_summary}
\end{table}

\section{Offline Contradiction Sanity Test}
As an additional sanity check, we evaluate ECR in a minimal fully offline scenario consisting of a single claim and its explicit synthetic negation (a paired contradiction twin). In this setting, the expected outcome is that ECR evaluates both claims, flags unresolved conflict (\texttt{has\_unresolved\_conflict=True}), and refuses to emit a dominant hypothesis (\texttt{dominant\_hypothesis=None}). This deterministic unit test passes.

\end{document}